%% file: KDD.tex
\documentclass[sigconf]{acmart}
\AtBeginDocument{%
  }

\copyrightyear{2026}
\acmYear{2026}
\setcopyright{cc}
\setcctype{by}
\acmConference[KDD 2026] {Proceedings of the 32nd ACM SIGKDD Conference on Knowledge Discovery and Data Mining V.2}{August 9--13, 2026}{Jeju Island, Republic of Korea.}
\acmBooktitle{Proceedings of the 32nd ACM SIGKDD Conference on Knowledge Discovery and Data Mining V.2 (KDD 2026), August 9--13, 2026, Jeju Island, Republic of Korea}
\acmISBN{979-8-4007-2259-2/2026/08}
\acmDOI{10.1145/3770855.3817797}

\usepackage{multirow} 
\usepackage{subfigure}
\usepackage{algorithm}
\usepackage{algorithmic}
\theoremstyle{plain}
\newtheorem{theorem}{Theorem}[section]

\newtheorem{corollary}[theorem]{Corollary}
\theoremstyle{definition}

\theoremstyle{remark}

\begin{document}

\title{Advancing Graph Few-Shot Learning via In-Context Learning}

\author{Renchu Guan}
\authornote{Key Laboratory of Symbolic Computation and Knowledge Engineering of the Ministry of Education}
\affiliation{
  \institution{College of Computer Science and Technology, Jilin University}
  \city{Changchun}
  \country{China}
}
\email{guanrenchu@jlu.edu.cn}

\author{Yajun Wang}
\authornotemark[1]
\affiliation{
  \institution{College of Computer Science and Technology, Jilin University}
  \city{Changchun}
  \country{China}
}
\email{yajun24@mails.jlu.edu.cn}

\author{Chunli Guo}
\authornotemark[1]
\affiliation{
  \institution{College of Software, Jilin University}
  \city{Changchun}
  \country{China}
}
\email{guocl24@mails.jlu.edu.cn}

\author{Bowen Cao}
\affiliation{
  \institution{Department of Computer Science and Technology, Yanbian University}
  \city{Yanbian}
  \country{China}
}
\email{caobowen@ybu.edu.cn}

\author{Fausto Giunchiglia}
\affiliation{%
  \institution{Department of Information Engineering and Computer Science, University of Trento}
  \city{Trento}
  \country{Italy}
}
\email{fausto.giunchiglia@unitn.it}

\author{Wei Pang}
\affiliation{%
  \institution{School of Mathematical and Computer Sciences, Heriot-Watt University}
  \city{Edinburgh}
  \country{Scotland}
}
\email{w.pang@hw.ac.uk}

\author{Yonghao Liu}
\authornotemark[1]
\authornote{Corresponding author}
\affiliation{
  \institution{College of Computer Science and Technology, Jilin University}
  \city{Changchun}
  \country{China}
}
\email{yonghao20@mails.jlu.edu.cn}

\author{Xiaoyue Feng}
\authornotemark[1]
\authornotemark[2]
\affiliation{%
  \institution{College of Computer Science and Technology, Jilin University}
  \city{Changchun}
  \country{China}
}
\email{fengxy@jlu.edu.cn}

\renewcommand{\shortauthors}{Renchu Guan et al.}
\begin{abstract}
Graph few-shot learning, which aims to classify nodes from novel classes with only a few labeled examples, is a widely studied problem in graph learning.
However, existing methods often face two key limitations. First, the predominant graph few-shot learning paradigm relies on supervised tasks, failing to leverage the vast number of unlabeled nodes in the graph.
Second, many approaches require complex task adaptation or fine-tuning during inference, limiting their efficiency and applicability.
Inspired by the powerful in-context learning capabilities of large language models, we propose a novel model named \textbf{VISION} for ad\textbf{V}anc\textbf{I}ng graph few-\textbf{S}hot learning via \textbf{I}n-c\textbf{O}ntext Lear\textbf{N}ing to address these challenges.
Our model reframes graph few-shot learning as a fine-tuning-free sequence reasoning problem.
At its core is a context-aware network that initializes nodes with role embeddings and employs a dual-context fusion module to synergistically integrate local topological structures and global task-level dependencies.
This allows our model to dynamically generate class-aware representations for the query set conditioned on the support set context in a single forward pass.
To effectively train our model, we introduce an unsupervised task generator that creates structure-adaptive features and constructs diverse pseudo-tasks from abundant unlabeled data.
Our method unifies unsupervised meta-learning with graph in-context learning, achieving efficient inference.
Extensive experiments on multiple benchmark datasets demonstrate the superiority of our model. Our public code can be found \textcolor{red}{\href{https://github.com/KEAML-JLU/VISION}{here}}.
\end{abstract}


\maketitle

\input{introduction}
\input{related_work}
\input{preliminary}
\input{method}
\input{theory}
\input{experiment}
\input{result}

\input{conclusion}

\begin{acks}
Our work is supported by the National Natural Science Foundation of China (No. 62372209). Fausto Giunchiglia's work is funded by European Union's Horizon 2020 FET Proactive Project (No.823783).
\end{acks}

\bibliographystyle{ACM-Reference-Format}
\balance
\bibliography{sample-base}

\appendix
\renewcommand{\thetable}{A.\arabic{table}}
\captionsetup[table]{labelformat=simple, labelsep=colon, name=Table}
\setcounter{table}{0}

\input{appendix}

\end{document}

%% file: introduction.tex
\section{Introduction}
\label{sec:introduction}
Graphs, as a fundamental data structure, are extensively utilized to represent a wide range of applications \cite{wu2020comprehensive, zhou2020graph, liu2025high}.
In recent years, graph neural networks (GNNs) have emerged as the \textit{de facto} standard for learning on graph-structured data due to their powerful representation capabilities.
However, the success of GNNs heavily relies on a large number of labeled nodes for training, a requirement that is often costly or infeasible in many real-world scenarios \cite{zhang2022few}, thus limiting their applicability.
Consequently, graph few-shot learning (FSL), which aims to enable rapid adaptation to novel classes using only a handful of labeled examples by leveraging knowledge from previously encountered tasks \cite{HuangZ20, liu2026graph}, has attracted significant research attention.
Despite notable progress, existing Graph FSL methods generally face two core challenges.

\textit{First}, most existing graph FSL models predominantly follow the episodic meta-learning paradigm, whose central idea is to train models by simulating few-shot tasks \cite{finn2017model}.
In this setting, each task is divided into a support set and a query set, where the model learns discriminative patterns from the support set and performs prediction on the query set \cite{snell2017prototypical}.
However, this paradigm does not explicitly exploit the contextual information shared across the support and query sets---such as topological relationships and class similarities---which can be highly beneficial for improving model performance \cite{wang2022graph}.
Notably, under this paradigm, models still require \textit{fine-tuning} with a few labeled data during inference, which inevitably introduces the risk of overfitting and hampers efficiency.

\textit{Second}, current graph FSL models typically rely on supervised meta-training.
This implies that these models are trained on a base set of classes with abundant labels, from which numerous supervised tasks are constructed to learn generalizable meta-knowledge.
Such approaches, however, present a fundamental contradiction: while aiming to solve data scarcity, they presuppose data richness during the training phase.
In practice, such an assumption is rarely attainable, as annotation is costly and often requires domain-specific expertise \cite{zhang2022few, tan2022transductive}.
For instance, in protein–protein interaction graphs, predicting compound properties typically relies on costly wet-lab experiments \cite{zhou2024proaffinity}.
More critically, this paradigm fails to effectively leverage the vast number of unlabeled nodes during meta-training.
The construction of learning tasks is confined to the labeled base classes, meaning the rich information within the sea of unlabeled data is not actively mined to diversify the learning experience, leaving its potential largely untapped.
This highlights the necessity of a training strategy that can bypass the need for base class labels and fully leverage the unlabeled data across the entire graph to guide model learning.

To address the aforementioned challenges, we propose a novel framework, \textbf{VISION}, ad\textbf{V}anc\textbf{I}ng graph few-\textbf{S}hot learning via \textbf{I}n-c\textbf{O}ntext Lear\textbf{N}ing.
Specifically, inspired by the powerful in-context learning capabilities inherent in the Transformer architecture of large language models \cite{singh2023transient, chan2022data}, we explore the transfer of this ability to the graph FSL setting.
To this end, we introduce a new context-aware paradigm that reformulates graph FSL as a fine-tuning-free sequence inference problem.
Our framework builds upon a context-aware network that explicitly models the task structure via role embeddings.
Instead of processing topological and task information sequentially, we design a dual-context fusion module that synergistically integrates local neighborhood structures and global task-level dependencies.
This enables the model to condition on the support set and dynamically generate class-aware representations for query nodes, thereby achieving a single pass forward without fine-tuning.

Moreover, to alleviate the reliance on labeled data during meta-training, we adopt an unsupervised strategy equipped with a pseudo-task generator.
This generator creates structure-adaptive features by dynamically gating node attributes with their neighborhood context, and subsequently constructs diverse pseudo-tasks---mirroring the structure of real tasks---based on similarity relations within this adaptive feature space.
Theoretically, our context-aware paradigm enjoys a tighter generalization error bound compared to conventional episodic meta-learning approaches.
Empirically, our model achieves superior performance across multiple benchmark datasets and demonstrates clear advantages over other competitive methods.

Our main contributions are summarized as follows:

\noindent (I) We propose \textbf{VISION}, a novel graph FSL framework based on a context-aware paradigm. It explicitly captures task-specific contextual information through a dual-context fusion mechanism, enabling effective in-context reasoning.

\noindent (II) We design an unsupervised pseudo-task generator that leverages structure-adaptive features to construct diverse tasks from unlabeled data, eliminating the dependency on labeled data during the meta-training stage.

\noindent (III) We conduct extensive experiments on multiple benchmark datasets, demonstrating the superiority of our proposed model against various baselines.

%% file: related_work.tex
\section{Related Work}
\subsection{Graph Neural Networks}
Graph Neural Networks (GNNs) have become a cornerstone in the analysis of graph-structured data, providing powerful support for graph representation learning \cite{kipf2016semi,zhou2020graph}. 
Building on the graph spectral theory, a series of graph convolutional networks subsequently emerged \cite{defferrard2016convolutional, henaff2015deep, kipf2016semi}, demonstrating strong learning performance through the design of diverse convolutional layers. 
In addition to spectral approaches, GNNs based on neighborhood aggregation mechanisms have also been widely explored. 
For instance, GraphSAGE \cite{hamilton2017inductive} generates node representations by sampling and aggregating neighbor features. Graph attention networks (GATs) \cite{velickovic2017graph} introduce trainable attention weights to enable fine-grained weighting of neighbors during aggregation. And graph isomorphism networks (GINs) \cite{xu2018powerful} extend this idea by applying arbitrary aggregation functions on multisets. 
Nevertheless, most existing GNN models primarily focus on supervised node classification and exhibit significant limitations in scenarios with extremely limited samples. 


\subsection{Graph Few-Shot Learning}
Recent studies have extended the scope of FSL to graph-structured data, giving rise to the field of graph FSL \cite{zhou2019meta, ding2020graph, liu2022few, liu2024meta, liu2024simple}. Conventional FSL methods, which are mostly developed for images or text, tend to perform unsatisfactorily when directly applied to graphs. To tackle this issue, GNNs, which are tailored for graph-structured information have been integrated with meta-learning frameworks. For instance, Meta-GNN \cite{zhou2019meta} incorporates MAML to adapt GNN models with very limited labeled data, while G-Meta \cite{huang2020graph} learns node representations through local subgraphs and enhances generalization via meta-learning. Other representative approaches include 
TENT \cite{wang2022task}, which constructs task-specific graph structures and adapts GNN parameters guided by class prototypes, and Meta-GPS \cite{liu2022few} adopts the scaling and shifting operations for the parameters to quickly adapt to the novel tasks. However, these methods largely adhere to the meta-learning paradigm while overlooking the contextual information within each task. 

\subsection{In-Context Learning}
In-context learning refers to the capability of completing novel tasks purely through inference, by conditioning on a small number of input–output demonstrations and generalizing to unseen queries \cite{dong2024survey}.
While this phenomenon is most prominently associated with large language models \cite{devlin2019bert,radford2019language,touvron2023llama}, recent studies have shown its applicability in a variety of domains, including image in-painting \cite{bar2022visual}, image segmentation \cite{butoi2023universeg}, and especially meta-learning \cite{chan2022data,singh2023transient}.
The GPICL \cite{kirsch2022general} framework further demonstrates that meta-training can enable transformers to function as versatile in-context learners.
Building on this line of work, CAML \cite{DBLP:conf/iclr/FiftyDJALRT24} adapts the in-context learning paradigm to address non-causal sequence modeling challenges.
Recently, a few pioneering studies have attempted to bridge the gap between in-context learning and graph learning. 
For instance, PRODIGY \cite{huang2023prodigy} introduces a pre-training objective based on neighbor matching to enable in-context learning on graphs. 
GraphPrompt \cite{liu2023graphprompt} employs prompt tuning mechanisms to unify pre-training and downstream tasks. 
However, these methods typically rely on extensive pre-training on large-scale datasets or complex prompt optimization stages to align objectives.
Consequently, exploring a lightweight and fine-tuning-free paradigm to adapt in-context learning to graph domains remains an open challenge.

%% file: preliminary.tex
\section{Preliminary Study}
In this work, we focus on the problem of few-shot node classification on a graph $\mathcal{G}=\{\mathcal{V}, \mathcal{E}, \mathbf{X}, \mathbf{A}\}$, where $\mathcal{V}$ and $\mathcal{E}$ denote the sets of nodes and edges, while $\mathbf{X}$ and $\mathbf{A}$ represent the feature and adjacency matrices, respectively.
The label space is partitioned into disjoint base classes $\mathcal{Y}_\text{base}$ and novel classes $\mathcal{Y}_\text{nov}$.
Following the standard episodic meta-learning paradigm, the model is trained and evaluated on a series of $N$-way $K$-shot tasks $\mathcal{T}=\{\mathcal{S}, \mathcal{Q}\}$, each comprising a support set $\mathcal{S}$ containing $K$ labeled examples for each of the $N$ classes, and a query set $\mathcal{Q}$ of instances to be classified.
Crucially, unlike conventional supervised approaches that rely on ground-truth labels from $\mathcal{Y}_\text{base}$, we adopt a strictly unsupervised meta-training strategy where the model learns transferable knowledge from self-supervised ``pseudo-tasks'' constructed solely from unlabeled data.
The ultimate goal is to enable the model to accurately predict the labels of the query set $\mathcal{Q}_\text{test}$ in a novel task by reasoning over the context provided by $\mathcal{S}_\text{test}$, without requiring any gradient-based fine-tuning.

%% file: method.tex
\section{Method}
\label{sec:method}
In this section, we provide a detailed description of our proposed VISION.
Our model is built upon two foundational pillars: an unsupervised meta-training strategy that enables learning from unlabeled data via a pseudo-task generator, and a novel context-aware network for in-context reasoning.
The former constructs diverse pseudo-tasks from the entire graph, driving the model to learn a generalizable classification ability without relying on base labeled data.
The latter reframes few-shot node classification as a sequence reasoning task, utilizing a dual-context fusion module to synergistically integrate structural and task contexts within a unified architecture to dynamically infer query labels.
We illustrate the overall framework of VISION in Fig.~\ref{fig:framework}.

\begin{figure*}[ht]
    \centering
    \includegraphics[width=0.96\linewidth]{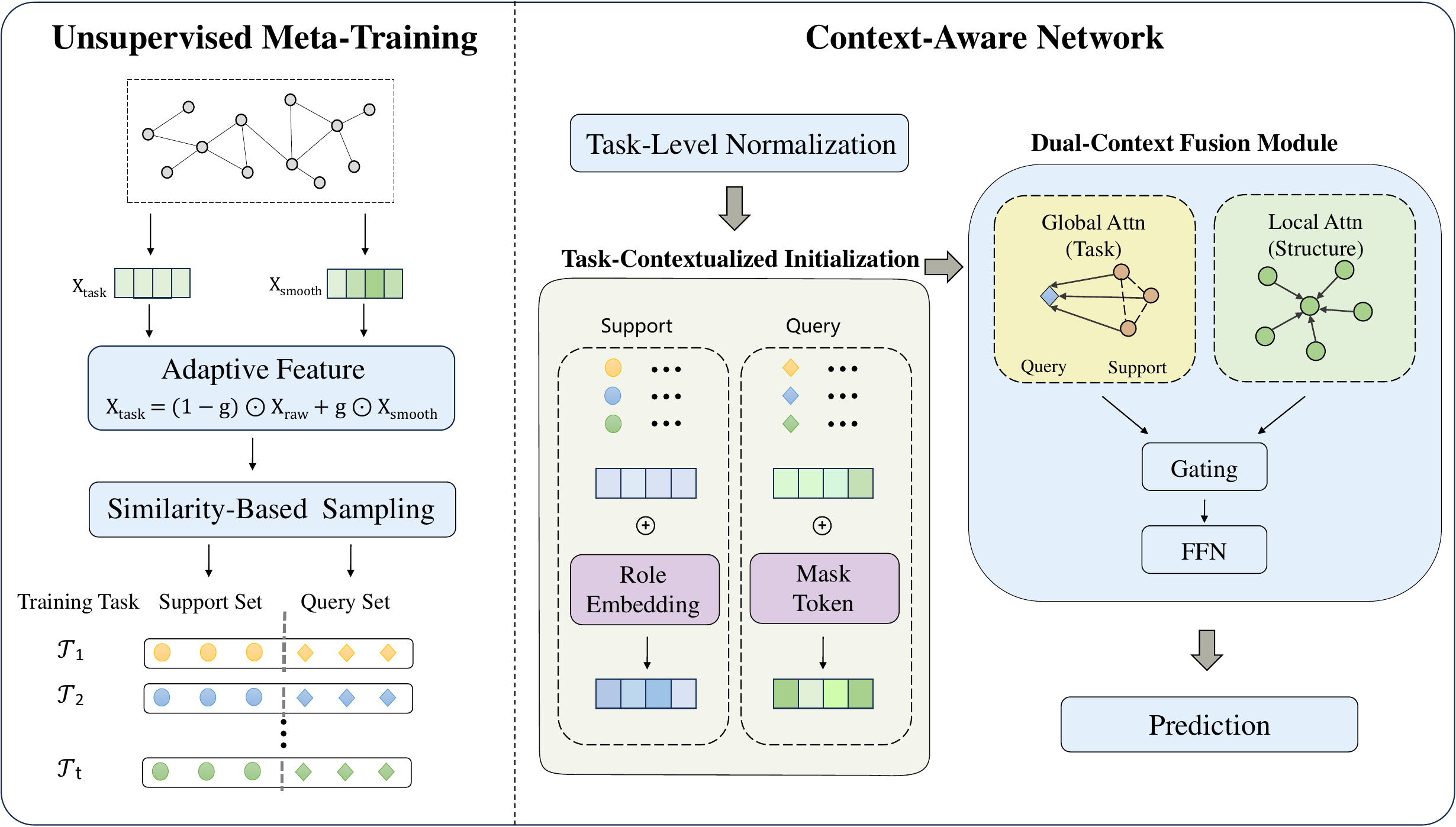}
    \caption{
        The overall framework of VISION.
        \textbf{Left:} The unsupervised meta-training module generates adaptive features $\mathbf{X}_{\text{task}}$ by fusing raw and smoothed features via an \textbf{adaptive gating mechanism}, then samples pseudo-tasks to simulate few-shot scenarios.
        \textbf{Right:} The context-aware network for inference. 
        It starts with the task-level normalization to mitigate distribution shifts, and then initializes input tokens via the task-contextualized initialization, which injects role embeddings (for support set) or mask tokens (for query set).
        The core dual-context fusion module operates in parallel: it simultaneously captures topological dependencies via local attention (structure) and task-specific relationships via global attention (task).
        These two complementary contexts are dynamically fused via gating to generate effective node representations.
        Finally, predictions are generated through a multi-head readout layer for ensemble prediction.
    }
    \label{fig:framework}
\end{figure*}

\subsection{Unsupervised Meta-Training}
\label{sec:unsupervised}
To overcome the dependency on base labeled data and fully leverage the rich unlabeled data in graphs, we propose an unsupervised meta-training strategy.
The core of this strategy is to first generate features for all nodes that adapt to the local topology, and then construct training tasks using a pseudo-task generator based on these more robust features.

\subsubsection{Structure-Adaptive Feature Generation}
To provide a more reliable node representation for the subsequent similarity-based sampling, we design an adaptive feature generation module.
The core idea is to dynamically balance a node's self-features with contextual information from its neighborhood.
We first define two fundamental types of features: the node's original features, $\mathbf{X}_{\text{raw}}$, representing its inherent attributes, and the smoothed features, $\mathbf{X}_{\text{smooth}}$, obtained via a single simplifying graph convolution (SGC) \cite{wu2019simplifying} step, which represents the average context of the node's one-hop neighborhood.
This is shown below: 
\begin{equation}
    \mathbf{X}_{\text{smooth}} = \tilde{\mathbf{D}}^{-\frac{1}{2}}\tilde{\mathbf{A}}\tilde{\mathbf{D}}^{-\frac{1}{2}}\mathbf{X}_{\text{raw}},
\end{equation}
where $\tilde{\mathbf{A}}$ is the adjacency matrix with self-loops and $\tilde{\mathbf{D}}_{ii}=\sum_j\tilde{\mathbf{A}}_{ij}$ is the corresponding diagonal degree matrix.
Notably, using only $\mathbf{X}_{\text{raw}}$ ignores structural information, while relying solely on $\mathbf{X}_{\text{smooth}}$ erases a node's distinctiveness.
Therefore, we introduce a gating mechanism to learn an optimal fusion ratio for each node.
This gate value, $\mathbf{g} \in \mathbb{R}^{N \times 1}$, is derived from local signal consistency. We measure this consistency for each node $i$ by calculating the cosine similarity score, $s_i$, between its raw feature vector $\mathbf{x}_{\text{raw}}^{(i)}$ and its smoothed feature vector $\mathbf{x}_{\text{smooth}}^{(i)}$, where both are L2-normalized. This score is then scaled to the range $[0, 1]$ to form the gate value $\mathbf{g}_i$:
\begin{equation}
\begin{aligned}
    & s_i = \text{cosine}(\mathbf{x}_{\text{raw}}^{(i)}, \mathbf{x}_{\text{smooth}}^{(i)}), \\
    & \mathbf{g}_i = (\text{clamp}(s_i, -1, 1) + 1) / 2,
\end{aligned}
\end{equation}
where the $\text{clamp}(s, a, b)$ function constrains the value $s$ within the inclusive range $[a, b]$.

From a graph signal processing perspective, this gating mechanism acts as a node-wise adaptive filter that explicitly handles the varying levels of homophily across the graph.
A high gate value identifies nodes in homophilous regions where neighborhood smoothing effectively denoises features.
Conversely, a low value protects heterophilous nodes or those on class boundaries from the over-smoothing phenomenon, preserving their distinctive high-frequency signals.
Finally, we use this gate to perform a weighted fusion of the two feature types, generating the final adaptive features $\mathbf{X}_{\text{task}}$:
\begin{equation}
    \mathbf{X}_{\text{task}} = (1 - \mathbf{g}) \odot \mathbf{X}_{\text{raw}} + \mathbf{g} \odot \mathbf{X}_{\text{smooth}},
\end{equation}
where $\odot$ denotes the Hadamard product (\textit{i.e.}, element-wise multiplication).
In this way, the resulting $\mathbf{X}_{\text{task}}$ is not a static representation but rather the result of a dynamic and fine-grained balance between a node's individuality and its neighborhood context.
This principled fusion provides a more robust feature foundation for the subsequent pseudo-task generation.

\subsubsection{Pseudo-Task Generation}
With the adaptive features, we can construct a multitude of pseudo-tasks that are structurally identical to real tasks.
Let $\mathcal{V}$ be the set of all nodes in the graph.
The process is as follows:

\noindent (1) \textbf{Anchor Selection}: $N$ nodes are randomly selected from the entire node set $\mathcal{V}$ to serve as ``anchors,'' which act as the ``class centers'' for the pseudo-task.
The rationale is that randomly sampled nodes are highly unlikely to belong to the same class, as demonstrated in \cite{siavash2019unsupervised}.
We mathematically and empirically validate this in \textbf{Appendix}~\ref{sec:anchor_verification}, showing high class distinctness across datasets.
Furthermore, we posit that occasional class collisions introduce a degree of label noise that increases task difficulty, thereby potentially enhancing the model's robustness against real-world ambiguity.
    
\noindent (2) \textbf{Nearest Neighbor Sampling}: For each anchor, we first sample a large candidate pool of nodes, $\mathcal{V}_{cand} \subset \mathcal{V}$, where $|\mathcal{V}_{cand}| = S_{pool}$. 
Subsequently, to ensure class discriminability, we employ a sequential exclusion sampling strategy to identify the $(K+M)$ most similar nodes to the current anchor from $\mathcal{V}_{cand}$.
Specifically, nodes selected for previous anchors are masked out to prevent class overlap.
    This similarity is measured by the cosine similarity in the adaptive feature space $\mathbf{X}_{\text{task}}$.
    
\noindent (3) \textbf{Task Construction}: These $N$ disjoint groups of sampled nodes together form an $N$-way $(K+M)$-shot pseudo-task.
We designate the first $K$ nodes in each group as the pseudo-support set $\mathcal{S}_{\text{pseudo}}$ and the remaining $M$ nodes as the pseudo-query set $\mathcal{Q}_{\text{pseudo}}$.

By training on a multitude of such pseudo-tasks, the model learns the general ability to solve the abstract task of ``classification'' rather than memorizing the features of specific classes.

\subsection{Context-Aware Network}
The context-aware network is designed to generate task-aware node representations by synergistically integrating structural and task contexts.
As illustrated in Fig. ~\ref{fig:framework}, the inference process begins with the task-level normalization.
Given the adaptive features $\mathbf{X}_{\text{task}}$ from the meta-training stage (or raw features during meta-testing), let $\mathcal{V}_{\mathcal{T}}$ denote the set of all nodes involved in the current task (comprising both support and query sets).
We first compute the task-level mean vector $\boldsymbol{\mu}_{\mathcal{T}}$ to capture the global distribution of the current episode:
\begin{equation}
    \boldsymbol{\mu}_{\mathcal{T}} = \frac{1}{|\mathcal{V}_{\mathcal{T}}|} \sum_{v \in \mathcal{V}_{\mathcal{T}}} \mathbf{x}_v,
\end{equation}
where $\mathbf{x}_v$ corresponds to the feature vector of node $v$ from $\mathbf{X}_{\text{task}}$.
Subsequently, we perform centering to obtain the normalized feature $\hat{\mathbf{x}}_v$ for each node $v \in \mathcal{V}_{\mathcal{T}}$:
\begin{equation}
    \hat{\mathbf{x}}_v = \mathbf{x}_v - \boldsymbol{\mu}_{\mathcal{T}}.
\end{equation}
This operation effectively mitigates the distribution shift between meta-training and meta-testing phases by aligning the feature distributions of different tasks to a common center.

Subsequently, we employ the task contextualized initialization to initialize the node representations.
Specifically, for a node $v$, let $\mathbf{x}_v \in \hat{\mathbf{X}}$ denote its normalized feature vector.
We project $\mathbf{x}_v$ into a hidden space and inject a learnable role embedding to denote its role within the task.
Crucially, to ensure generalization across tasks with disjoint label sets, we re-index the original labels of the support set to a local relative range $\{0, \dots, N-1\}$ within each episode.
Thus, for nodes in the support set with relative label $y_v$, we add a corresponding role embedding $\mathbf{e}_{y_v}$;
for query nodes (whose labels are masked), we add a shared mask token $\mathbf{e}_{\text{mask}}$.
It is worth noting that this design strictly preserves the \textit{permutation invariance} of the graph data.
Since the role embedding $\mathbf{e}_{y_v}$ is assigned solely based on the node's relative class identity (\textit{i.e.}, $y_v$) rather than its positional index in the input sequence, permuting the order of nodes in the support set does not alter their representations.
This ensures that the model focuses on capturing class-level commonalities rather than memorizing sequence patterns.
The initial input token $\mathbf{h}_v^{(0)}$ is defined as follows:
\begin{equation}
\label{eq:input_embedding}
    \mathbf{h}_v^{(0)} = \text{LayerNorm}(\mathbf{W}_f \hat{\mathbf{x}}_v + \mathbf{e}_{\text{role}}),
\end{equation}
where $\mathbf{e}_{\text{role}} \in \{\mathbf{e}_{0}, \dots, \mathbf{e}_{N-1}, \mathbf{e}_{\text{mask}}\}$.
This role-enhanced initialization allows the model to explicitly distinguish between support and query roles from the onset, independent of the specific semantic classes.
Building on these initialized tokens, the network employs two key components: the dual-context fusion module for parallel reasoning and the multi-head readout for robust prediction.

\subsubsection{Dual-Context Fusion Module}
\label{sec:dual_context}
The core of our architecture is the dual-context fusion module, which operates in a decoupled, parallel manner to capture complementary contexts.
Unlike previous methods that process information sequentially, this module simultaneously computes two distinct attention views:

\noindent \textbf{Local Attn (Structure).}
To capture the topological dependencies, we employ a structure-aware attention mechanism.
For a central node $v$, we sample a fixed-size subset of neighbors $\mathcal{N}_{\text{s}}(v)$ with a maximum size of $k_{\text{neigh}}$.
Node $v$ serves as the query, attending to its neighborhood sequence to aggregate local structural evidence.
Formally, the local representation $\mathbf{z}_{\text{struct},v}$ is computed as follows:
\begin{equation}
\label{eq:local_view}
    \mathbf{z}_{\text{struct},v} = \text{Attention}(\mathbf{h}_v \mathbf{W}_Q^L, \{\mathbf{h}_u \mathbf{W}_{key}^L\}_{u \in \mathcal{N}_{\text{s}}(v)}, \{\mathbf{h}_u \mathbf{W}_V^L\}_{u \in \mathcal{N}_{\text{s}}(v)}).
\end{equation}

\noindent \textbf{Global Attn (Task).}
Simultaneously, to capture the task-specific relationships, each node needs to understand its affinity with the support set $\mathcal{S}$.
We employ a global cross-attention mechanism where node $v$ acts as the query and the entire support set $\mathcal{S}$ acts as the key and value.
This allows query nodes to directly fine-tuning-free reason about their similarity to the labeled examples, generating the task-level representation $\mathbf{z}_{\text{task},v}$:
\begin{equation}
\label{eq:global_view}
    \mathbf{z}_{\text{task},v} = \text{Attention}(\mathbf{h}_v \mathbf{W}_Q^G, \{\mathbf{h}_k \mathbf{W}_{key}^G\}_{k \in \mathcal{S}}, \{\mathbf{h}_k \mathbf{W}_V^G\}_{k \in \mathcal{S}}).
\end{equation}

To dynamically resolve the potential tension between task-agnostic topology and task-specific semantics, we introduce a conflict-aware gating mechanism.
This mechanism allows the model to selectively suppress structural noise when the local neighborhood is task-irrelevant, or conversely, to rely on topological evidence when the few-shot task signals are ambiguous.
We calculate a scalar gate $\alpha_v \in [0, 1]$ based on the concatenation of both representations:
\begin{equation}
\label{eq:gating}
\begin{aligned}
    \alpha_v &= \sigma(\mathbf{W}_g [\mathbf{z}_{\text{struct},v} || \mathbf{z}_{\text{task},v}]), \\
    \mathbf{h}'_v &= (1 - \alpha_v) \odot \mathbf{z}_{\text{struct},v} + \alpha_v \odot \mathbf{z}_{\text{task},v}.
\end{aligned}
\end{equation}
The fused representation $\mathbf{h}'_v$ is then processed by a Feed-Forward Network (FFN) to generate the task-aware node embedding $\mathbf{Z}_{\text{final}}$.

\subsubsection{Multi-Head Readout}
\label{sec:readout}
Finally, to mitigate the variance inherent in few-shot scenarios and improve robustness, we employ a multi-head readout strategy.
Instead of relying on a single metric space, we project the final embeddings into $H$ diverse subspaces.
In each subspace $h$, we compute the scaled cosine similarity between the query node representation $\mathbf{z}_q^h$ and the support prototype $\mathbf{p}_k^h$ for class $k$. 
Let $\mathbf s_{q,k}^h$ denote the specific logit for head $h$:
\begin{equation}
    \mathbf s_{q,k}^h = \tau_h \cdot \text{cosine}(\mathbf{z}_q^h, \mathbf{p}_k^h),
\end{equation}
where $\tau_h$ is a learnable temperature parameter for the $h$-th head.

The final prediction score $\mathbf S_{q,k}$ is then obtained by averaging the logits across all $H$ heads:
\begin{equation}
\label{eq:readout}
    \mathbf S_{q,k} = \frac{1}{H} \sum_{h=1}^{H} \mathbf s_{q,k}^h.
\end{equation}


\subsection{Model Optimization}
\label{sec:optimization}
During the meta-training phase, we adopt a joint optimization strategy that combines a classification loss with an auxiliary contrastive loss to learn robust task-aware representations.
The total objective function is defined as:
\begin{equation}
    \mathcal{L}_{\text{total}} = \mathcal{L}_{\text{CE}} + \lambda \cdot \mathcal{L}_{\text{Con}},
\end{equation}
where $\lambda$ is a hyperparameter balancing the two terms.

\noindent \textbf{Classification Loss.}
The primary objective is to minimize the prediction error on the pseudo-query sets.
We utilize the final ensemble scores $\mathbf S_{q,k}$ obtained from Eq.~\ref{eq:readout}.
To prevent overfitting and encourage better generalization, we employ the cross-entropy loss with label smoothing:
\begin{equation}
    \mathcal{L}_{\text{CE}} = -\sum_{(\mathbf{z}_q, y_q) \in \mathcal{Q}_{\text{pseudo}}} \sum_{k=1}^{N} y_{q,k}^{\text{ls}} \log \left( \frac{\exp(\mathbf S_{q,k})}{\sum_{j=1}^{N} \exp(\mathbf S_{q,j})} \right),
\end{equation}
where $y_{q,k}^{\text{ls}}$ is the smoothed label for class $k$.

\noindent \textbf{Contrastive Loss.}
To further explicitly align the query samples with their corresponding support prototypes in the feature space, we introduce an auxiliary contrastive loss.
For each head $h$, we maximize the similarity between query node embeddings $\mathbf{z}_q^h$ and the support embeddings $\mathbf{z}_s^h$ of the same class, while minimizing it for different classes:
\begin{equation}
    \mathcal{L}_{\text{Con}} = -\frac{1}{H |\mathcal{Q}|} \sum_{h=1}^{H} \sum_{\mathbf{z}_q \in \mathcal{Q}} \log \frac{\sum_{\mathbf{z}_s \in \mathcal{S}_{y_q}} \exp(\mathbf{z}_q^h \cdot \mathbf{z}_s^h / \tau)}{\sum_{\mathbf{z}_{s'} \in \mathcal{S}} \exp(\mathbf{z}_q^h \cdot \mathbf{z}_{s'}^h / \tau)}.
\end{equation}

\noindent \textbf{Meta-Testing Inference.}
During the meta-testing phase, VISION is entirely fine-tuning-free.
For a novel task, we perform a single forward pass.
Thanks to the integrated multi-head readout, our model directly outputs the ensemble similarity scores (logits) for each query node.
The predicted label $\hat{y}_q$ is simply the index with the maximum score:
\begin{equation}
\label{eq:inference}
    \hat{y}_q = \underset{k}{\arg\max} \ \mathbf S_{q, k}.
\end{equation}

Our training algorithm is summarized in Algorithm~\ref{alg:vision_overall}. 

\begin{algorithm}[ht]
\caption{The Overall Procedure of VISION.}
\label{alg:vision_overall}
\begin{algorithmic}[1]
\REQUIRE A graph $\mathcal{G}=\{\mathcal{V}, \mathcal{E}, \mathbf{X}, \mathbf{A}\}$.
\ENSURE The trained VISION model.
\STATE // \textit{Phase 1: Unsupervised Meta-Training}
\STATE Generate structure-adaptive features $\mathbf{X}_{\text{task}}$ by fusing raw and smoothed features.
\COMMENT{Pre-computation}
\WHILE{not converged}
    \STATE Construct a pseudo-task $\mathcal{T}_{\text{pseudo}}=\{\mathcal{S}_{\text{pseudo}}, \mathcal{Q}_{\text{pseudo}}\}$ using adaptive features.
    \STATE \textbf{Task-Level Normalization:} Compute task mean and center features for all nodes in $\mathcal{T}_{\text{pseudo}}$.
    \STATE \textbf{Task-Contextualized Initialization:} Initialize $\mathbf{H}^{(0)}$ with role injection (Eq.~\ref{eq:input_embedding}).
    \STATE \textbf{Dual-Context Fusion:} Compute local structure $\mathbf{Z}_{\text{struct}}$ and global task $\mathbf{Z}_{\text{task}}$ contexts in parallel, then fuse them via gating (Eq.~\ref{eq:gating}).
    \STATE \textbf{Multi-Head Readout:} Generate ensemble logits $\mathbf{S}_q$ (Eq.~\ref{eq:readout}).
    \STATE Compute joint loss $\mathcal{L}_{\text{total}} = \mathcal{L}_{\text{CE}} + \lambda \mathcal{L}_{\text{Con}}$.
    \STATE Optimize model parameters by minimizing $\mathcal{L}_{\text{total}}$.
\ENDWHILE
\STATE // \textit{Phase 2: Fine-tuning-free Meta-Testing}
\STATE \textbf{Input:} A novel meta-test task $\mathcal{T}_{\text{test}}=\{\mathcal{S}_{\text{test}}, \mathcal{Q}_{\text{test}}\}$.
\STATE Perform a single forward pass (Steps 5-8 above) on $\mathcal{T}_{\text{test}}$.
\STATE Predict labels $\hat{y}_q = \arg\max \mathbf{S}_q$ for query nodes.
\STATE \textbf{Output:} Predicted labels $\hat{Y}_{\mathcal{Q}}$.
\end{algorithmic}
\end{algorithm}

%% file: theory.tex
\section{Theoretical Analysis}
\label{sec:theory}

In this section, we provide a theoretical analysis of why VISION, trained on unsupervised pseudo-tasks, can generalize to real-world few-shot tasks without fine-tuning. We analyze the generalization error $\epsilon_{gen} = \mathbb{E}_{\mathcal{T} \sim P_{real}}[\mathcal{L}(f, \mathcal{T})]$ within a domain adaptation framework, considering the pseudo-task distribution as the \textit{source domain} and the real task distribution as the \textit{target domain}.

\begin{theorem}[Generalization Bound via Domain Alignment]
\label{thm:main}
Let $\mathcal{H}$ be the hypothesis space of the model $f$. Assume the loss function $\mathcal{L}$ is bounded and $L$-Lipschitz continuous with respect to the input features.
Let $\hat{f}$ be the model minimizing the empirical risk on $T$ sampled pseudo-tasks.
With probability at least $1-\eta$, the generalization error on real tasks is bounded by:
\begin{equation}
\label{eq:main_bound}
\epsilon_{gen}(\hat{f}) \leq \underbrace{\hat{\epsilon}_{pseudo}(\hat{f}) + \mathcal{R}_T(\mathcal{H})}_{\text{(A) Meta-Learning Error}} + \underbrace{L \cdot W_1(P_{pseudo}^{\phi}, P_{real}^{\phi})}_{\text{(B) Distribution Alignment Error}},
\end{equation}
where $\mathcal{R}_T(\mathcal{H})$ is the Rademacher complexity (model capacity term), $W_1(\cdot, \cdot)$ is the Wasserstein-1 distance between the feature distributions of pseudo and real tasks \textit{after} Task-Level Normalization ($\phi$), and $\lambda$ is the optimal joint error (assumed small).
\end{theorem}

\noindent \textbf{Remark on Term (B):} The core challenge in our unsupervised setting is the domain gap (Term B).
Existing methods often assume this gap is negligible, which is a strong assumption.
In VISION, we explicitly minimize this term via two mechanisms:
(1) \textit{Structure-Adaptive Sampling} ensures the conditional probability $P(Y|G_{struct})$ in pseudo-tasks mimics the homophily property of real networks.
(2) \textit{Task-Level Normalization}, implemented in our code, aligns the marginal feature distributions by centering task-specific statistics, theoretically bounding $W_1(P_{pseudo}^{\phi}, P_{real}^{\phi})$.

\begin{corollary}[Benefit of In-Context Reasoning]
\label{coro:icl}
Consider the predictive variance of a query node representation $\mathbf{z}_q$. Let $\text{Var}_{local}$ be the variance of the representation derived solely from local neighbors (baseline), and $\text{Var}_{VISION}$ be the variance after the global task attention module processes the support set $\mathcal{S}$. Assuming the support set provides independent observations of the class prototype with noise variance $\sigma^2$, we have:
\begin{equation}
\text{Var}_{VISION}(\mathbf{z}_q) \approx \text{Var}_{local}(\mathbf{z}_q) - \frac{\gamma}{K} \sigma^2 < \text{Var}_{local}(\mathbf{z}_q),
\end{equation}
where $K$ is the shot number and $\gamma$ is the attention gain factor.
\end{corollary}

Corollary \ref{coro:icl} justifies our design: the global context module acts as a Bayesian posterior refinement step. By conditioning on $K$ support examples in a single forward pass, it explicitly reduces the epistemic uncertainty of the query prediction, a benefit that local-only GNNs cannot achieve.

The detailed proofs for Theorem \ref{thm:main} and Corollary \ref{coro:icl} are provided in \textbf{Appendix} \ref{sec:appendix_proofs}.

%% file: experiment.tex
\section{Experiments}
\subsection{Datasets}
To evaluate the effectiveness of the proposed model, we conduct experiments on several benchmark datasets commonly used for few-shot node classification, including \textbf{CoraFull} \cite{bojchevski2017deep}, \textbf{Cora} \cite{yang2016revisiting}, \textbf{Coauthor-CS} \cite{shchur2018pitfalls}, \textbf{CiteSeer} \cite{yang2016revisiting}, \textbf{ogbn-arxiv} \cite{hu2020open},  and \textbf{Cora-ML} \cite{bojchevski2017deep}. The statistics of these datasets are presented in Table \ref{datasets}. We present descriptions of these datasets in \textbf{Appendix} \ref{Dataset_Descriptions}.

\begin{table}[ht]
\centering
\caption{Statistics of the datasets used in our experiments.}
\begin{tabular}{lcccc}
\toprule
\textbf{Dataset} & \textbf{\# Nodes} & \textbf{\# Edges} & \textbf{\# Features} & \textbf{\# Labels} \\
\midrule
CoraFull       & 19,793   & 65,311      & 8,710  & 70 \\
Cora           & 2,708    & 5,278       & 1,433  & 7  \\
Coauthor-CS    & 18,333   & 81,894      & 6,805  & 15 \\
CiteSeer       & 3,327    & 4,552       & 3,703  & 6  \\
ogbn-arxiv     & 169,343  & 1,166,243   & 128    & 40 \\
Cora-ML             & 2,995    & 16,316      & 2,879  & 7  \\
\bottomrule
\end{tabular}
\label{datasets}
\end{table}

\subsection{Baselines}
To demonstrate the superiority of our proposed model, we conduct comparisons with baselines from four representative categories. \textit{Graph embedding methods} contain \textbf{DeepWalk} \cite{perozzi2014deepwalk}, \textbf{node2vec} \cite{grover2016node2vec}, \textbf{GCN} \cite{kipf2016semi}, and \textbf{SGC} \cite{wu2019simplifying}. \textit{Meta-learning methods} include \textbf{ProtoNet} \cite{snell2017prototypical} and \textbf{MAML} \cite{finn2017model}. \textit{Graph in-context learning methods} contain \textbf{VNT} \cite{tan2023virtual} and \textbf{GraphPrompt} \cite{liu2023graphprompt}. Since PRODIGY \cite{huang2023prodigy} does not release its pretrained weights, we do not include it as a baseline. \textit{Graph meta-learning methods} consist of \textbf{GPN} \cite{ding2020graph}, \textbf{G-Meta} \cite{huang2020graph}, \textbf{TENT} \cite{wang2022task}, \textbf{Meta-GPS} \cite{liu2022few}, \textbf{TEG} \cite{kim2023task}, \textbf{COSMIC} \cite{wang2023contrastive}, \textbf{TLP} \cite{tan2022transductive}, \textbf{Meta-BP} \cite{zhang2025unlocking}, and \textbf{SMILE} \cite{liu2025dual}. The descriptions of these baselines are presented in \textbf{Appendix} \ref{Baseline_Descriptions}.

\subsection{Implementation Details}
We implement the proposed context-aware architecture with a hidden dimension of 256, four attention heads, two Transformer layers, and a feed-forward network dimension of 512. To enhance robustness, we employ a multi-head readout strategy with three ensemble heads and set the maximum number of sampled neighbors $k_{\text{neigh}}=30$. The model is optimized using the AdamW optimizer \cite{DBLP:conf/iclr/LoshchilovH19} with an initial learning rate of 2e-4, a weight decay of 1e-4, and a cosine annealing scheduler. To prevent overfitting, we incorporate input noise injection with a standard deviation of $\sigma=0.02$. Pseudo-tasks are constructed by sampling from a candidate pool of 4,096 nodes. To ensure reproducibility and fair comparison, we use a fixed random seed for both our model and all baselines. For datasets with a large number of classes, such as CoraFull and ogbn-arxiv, we adopt 5-way 3-shot, 5-way 5-shot, 10-way 3-shot, and 10-way 5-shot few-shot settings. For datasets with a moderate number of classes, such as Coauthor-CS, we consider 2-way 3-shot, 2-way 5-shot, 5-way 3-shot, and 5-way 5-shot settings. For the remaining datasets with fewer classes, we consistently use 2-way 1-shot, 2-way 3-shot, and 2-way 5-shot experimental settings. We report the average results of 5 independent experiments (100 episodes each) carried out on a single NVIDIA 3090 GPU.

%% file: result.tex
\section{Results}
\subsection{Model Performance}

\begin{table*}[ht]
  \caption{Accuracies (\%) on CoraFull, Cora, and Coauthor-CS datasets. Bold: best, Underline: runner-up.}
  \label{results-merged-1}
  \centering
  \resizebox{0.96\textwidth}{!}{%
  \begin{tabular}{l|cccc|ccc|cccc}
    \toprule
    \multirow{2}{*}{\textbf{Model}} & \multicolumn{4}{c}{\textbf{CoraFull}} & \multicolumn{3}{c}{\textbf{Cora}} & \multicolumn{4}{c}{\textbf{Coauthor-CS}} \\
    \cmidrule(lr){2-5} \cmidrule(lr){6-8} \cmidrule(lr){9-12}
    & 5 way 3 shot & 5 way 5 shot & 10 way 3 shot & 10 way 5 shot & 2 way 1 shot & 2 way 3 shot & 2 way 5 shot & 2 way 3 shot & 2 way 5 shot & 5 way 3 shot & 5 way 5 shot \\
    \midrule
    DeepWalk    & 23.62$\pm$3.99 & 25.93$\pm$3.45 & 15.32$\pm$4.12 & 17.03$\pm$3.73 & 32.95$\pm$2.70 & 36.70$\pm$2.99 & 41.51$\pm$2.70 & 59.52$\pm$2.72 & 63.12$\pm$3.12 & 33.76$\pm$3.21 & 40.15$\pm$2.96 \\
    node2vec    & 23.75$\pm$2.93 & 25.42$\pm$3.61 & 13.90$\pm$3.32 & 15.21$\pm$2.64 & 31.17$\pm$3.16 & 35.66$\pm$2.79 & 40.69$\pm$2.90 & 56.16$\pm$4.19 & 60.22$\pm$4.06 & 30.35$\pm$3.93 & 39.16$\pm$3.79 \\
    GCN         & 34.65$\pm$2.76 & 39.83$\pm$2.49 & 29.23$\pm$3.25 & 34.14$\pm$2.15 & 55.46$\pm$2.16 & 69.96$\pm$2.52 & 67.95$\pm$2.36 & 73.52$\pm$1.97 & 77.20$\pm$3.01 & 52.19$\pm$2.31 & 56.35$\pm$2.99 \\
    SGC         & 39.56$\pm$3.52 & 44.53$\pm$2.92 & 35.12$\pm$2.71 & 39.53$\pm$3.32 & 56.75$\pm$2.31 & 70.15$\pm$1.99 & 70.67$\pm$2.11 & 75.49$\pm$2.15 & 79.63$\pm$2.01 & 56.39$\pm$2.26 & 59.25$\pm$2.16 \\
    \midrule
    ProtoNet    & 33.67$\pm$2.51 & 36.53$\pm$3.76 & 24.90$\pm$2.03 & 27.24$\pm$2.67 & 50.39$\pm$2.52 & 52.67$\pm$2.28 & 57.92$\pm$2.34 & 71.18$\pm$3.82 & 75.51$\pm$3.19 & 47.71$\pm$3.92 & 51.66$\pm$2.51 \\
    MAML        & 37.12$\pm$3.16 & 47.51$\pm$3.09 & 26.61$\pm$2.19 & 31.60$\pm$2.91 & 52.40$\pm$2.29 & 55.07$\pm$2.36 & 57.39$\pm$2.23 & 62.32$\pm$4.60 & 65.20$\pm$4.20 & 36.99$\pm$4.32 & 42.12$\pm$2.43 \\
    \midrule
    VNT  & 55.19$\pm$2.59 &
  70.16$\pm$1.36 &
  45.76$\pm$1.72 &
  57.92$\pm$2.26  & 59.39$\pm$2.22 & 75.32$\pm$1.96 & 79.39$\pm$2.49 &   86.12$\pm$3.20 &
  89.95$\pm$3.29 &
  80.16$\pm$2.55 &
  82.92$\pm$2.36 \\
    GraphPrompt‌ & 65.12$\pm$2.05 & 70.25$\pm$2.26 & 59.32$\pm$2.35 & 62.19$\pm$2.09 & 61.55$\pm$2.42 & 77.17$\pm$1.50 & \underline{85.22$\pm$1.99} &  89.22$\pm$2.95 & 91.39$\pm$3.22 & 81.63$\pm$2.51 & 84.35$\pm$2.90\\
    \midrule
    Meta-GNN    & 52.23$\pm$2.41 & 59.12$\pm$2.36 & 47.21$\pm$3.06 & 53.32$\pm$3.15 & 58.82$\pm$2.56 & 70.40$\pm$2.64 & 72.51$\pm$1.91 & 85.76$\pm$2.74 & 87.86$\pm$4.79 & 75.87$\pm$3.88 & 68.59$\pm$2.59 \\
    GPN         & 53.24$\pm$2.33 & 60.31$\pm$2.19 & 50.93$\pm$2.30 & 56.21$\pm$2.09 & 60.12$\pm$2.12 & 74.05$\pm$1.96 & 76.39$\pm$2.33 & 85.60$\pm$2.15 & 88.70$\pm$2.21 & 75.88$\pm$2.75 & 81.79$\pm$3.18 \\
    G-Meta      & 57.52$\pm$3.91 & 62.43$\pm$3.11 & 53.92$\pm$2.91 & 58.10$\pm$3.02 & 59.72$\pm$3.15 & 74.39$\pm$2.69 & 80.05$\pm$1.98 & 92.14$\pm$3.90 & \underline{93.90$\pm$3.18} & 75.72$\pm$3.59 & 74.18$\pm$3.29 \\
    TENT        & 64.80$\pm$4.10 & 69.24$\pm$4.49 & 51.73$\pm$4.34 & 56.00$\pm$3.53 & 55.39$\pm$2.16 & 58.25$\pm$2.23 & 66.75$\pm$2.19 & 89.35$\pm$4.49 & 90.90$\pm$4.24 & 78.38$\pm$5.21 & 78.56$\pm$4.42 \\
    Meta-GPS    & 65.19$\pm$2.35 & 69.25$\pm$2.52 & 61.23$\pm$3.11 & 64.22$\pm$2.66 & 62.19$\pm$2.12 & 80.29$\pm$2.15 & 83.79$\pm$2.10 & 90.16$\pm$2.72 & 92.39$\pm$1.66 & 81.39$\pm$2.35 & 83.66$\pm$1.79 \\
    X-FNC       & 69.32$\pm$3.10 & 71.26$\pm$4.19 & 49.63$\pm$4.45 & 53.00$\pm$3.93 & 61.47$\pm$2.99 & 78.19$\pm$3.25 & 82.70$\pm$3.19 & 90.95$\pm$4.29 & 92.03$\pm$4.14 & \underline{82.93$\pm$2.02} & 84.36$\pm$3.49 \\
    TEG         & 72.14$\pm$1.06 & 76.20$\pm$1.39 & 61.03$\pm$1.13 & 65.56$\pm$1.03 & 62.52$\pm$2.95 & \underline{80.65$\pm$1.53} & 84.50$\pm$2.01 & \underline{92.36$\pm$1.59} & 93.02$\pm$1.24 & 80.78$\pm$1.40 & 84.70$\pm$1.42 \\
    COSMIC      & 73.03$\pm$1.78 & 77.24$\pm$1.52 & \underline{65.79$\pm$1.36} & \underline{70.06$\pm$1.93} & 63.16$\pm$2.47 & 65.37$\pm$2.49 & 69.10$\pm$2.30 & 89.35$\pm$4.49 & 93.32$\pm$1.93 & 78.38$\pm$5.21 & \underline{85.47$\pm$1.42} \\
    TLP         & 66.32$\pm$2.10 & 71.36$\pm$4.49 & 51.73$\pm$4.34 & 56.00$\pm$3.53 & 60.19$\pm$2.25 & 71.10$\pm$1.66 & 85.15$\pm$2.19 & 90.35$\pm$4.49 & 90.90$\pm$4.24 & 82.30$\pm$2.05 & 78.56$\pm$4.42 \\
    Meta-BP     & 72.90$\pm$1.90 & 74.36$\pm$2.19 & 62.35$\pm$2.27 & 67.26$\pm$2.59 & 66.42$\pm$4.12 & 76.32$\pm$4.30 & 83.12$\pm$4.16 & 91.19$\pm$2.21 & 92.32$\pm$2.11 & 81.35$\pm$2.02 & 82.12$\pm$2.15 \\
    SMILE                           & \underline{74.04$\pm$1.88}              & \underline{77.30$\pm$1.61}              & 61.79$\pm$1.46              & 65.70$\pm$1.08              & \underline{71.90$\pm$3.54}              & 79.22$\pm$1.26 & 82.29$\pm$1.42              & 89.45$\pm$1.86              & 92.19$\pm$2.14 & 76.18$\pm$2.21 & 83.56$\pm$2.42              \\
    \midrule
    \textbf{VISION} & \textbf{80.23$\pm$1.80}    & \textbf{83.02$\pm$1.48}    & \textbf{70.73$\pm$1.32}    & \textbf{73.64$\pm$1.10}    & \textbf{76.11$\pm$2.80}    & \textbf{82.72$\pm$1.88}    & \textbf{86.02$\pm$1.63}    & \textbf{94.14$\pm$1.39}    & \textbf{95.34$\pm$1.22}    & \textbf{93.20$\pm$0.74}    & \textbf{94.56$\pm$0.64} \\
    \bottomrule
  \end{tabular}%
  }
\end{table*}

\begin{table*}[ht]
  \caption{Accuracies (\%) on CiteSeer, ogbn-arxiv, and Cora-ML datasets. Bold: best, Underline: runner-up.}
  \label{results-merged-2}
  \centering
  \resizebox{0.96\textwidth}{!}{%
  \begin{tabular}{l|ccc|cccc|ccc}
    \toprule
    \multirow{2}{*}{\textbf{Model}} & \multicolumn{3}{c}{\textbf{CiteSeer}} & \multicolumn{4}{c}{\textbf{ogbn-arxiv}} & \multicolumn{3}{c}{\textbf{Cora-ML}} \\
    \cmidrule(lr){2-4} \cmidrule(lr){5-8} \cmidrule(lr){9-11}
    & 2 way 1 shot & 2 way 3 shot & 2 way 5 shot & 5 way 3 shot & 5 way 5 shot & 10 way 3 shot & 10 way 5 shot & 2 way 1 shot & 2 way 3 shot & 2 way 5 shot \\
    \midrule
    DeepWalk    & 39.56$\pm$2.79 & 39.72$\pm$3.42 & 43.22$\pm$3.19 & 24.12$\pm$3.16 & 26.16$\pm$2.95 & 20.19$\pm$2.35 & 23.76$\pm$3.02 & 38.15$\pm$2.92 & 40.32$\pm$3.10 & 42.19$\pm$2.55 \\
    node2vec    & 40.12$\pm$3.15 & 42.39$\pm$2.79 & 47.20$\pm$2.92 & 25.29$\pm$2.96 & 27.39$\pm$2.56 & 22.99$\pm$3.15 & 25.95$\pm$3.12 & 41.12$\pm$3.11 & 43.19$\pm$2.17 & 45.16$\pm$2.62 \\
    GCN         & 51.95$\pm$2.45 & 53.79$\pm$2.39 & 55.76$\pm$2.56 & 32.26$\pm$2.11 & 36.29$\pm$2.39 & 30.21$\pm$1.95 & 33.96$\pm$1.59 & 62.12$\pm$2.35 & 67.45$\pm$2.11 & 69.22$\pm$2.01 \\
    SGC         & 53.72$\pm$2.55 & 55.12$\pm$2.59 & 57.25$\pm$2.79 & 35.19$\pm$2.76 & 39.76$\pm$2.95 & 31.99$\pm$2.12 & 35.22$\pm$2.52 & 61.39$\pm$2.59 & 69.51$\pm$3.10 & 70.29$\pm$1.95 \\
    \midrule
    ProtoNet    & 49.15$\pm$2.29 & 52.19$\pm$2.96 & 53.75$\pm$2.49 & 39.99$\pm$3.28 & 47.31$\pm$1.71 & 32.79$\pm$2.22 & 37.19$\pm$1.92 & 63.12$\pm$2.12 & 68.40$\pm$2.60 & 76.94$\pm$2.39 \\
    MAML        & 49.15$\pm$2.25 & 52.75$\pm$2.75 & 54.36$\pm$2.39 & 28.35$\pm$1.49 & 29.09$\pm$1.62 & 30.19$\pm$2.97 & 36.19$\pm$2.29 & 62.19$\pm$3.65 & 68.61$\pm$2.51 & 71.27$\pm$2.62 \\
    \midrule
    VNT & 59.72$\pm$2.25 & 69.25$\pm$2.16 & 73.90$\pm$2.10 & OOM & OOM & OOM & OOM & 74.59$\pm$2.12 &86.22$\pm$2.45 & 89.26$\pm$1.76 \\
    GraphPrompt‌ & 57.59$\pm$2.12  & 67.49$\pm$2.11 & 70.59$\pm$2.09 & 51.52$\pm$2.55 &62.25$\pm$2.66 & 40.19$\pm$2.02 & 45.39$\pm$2.12 & 75.12$\pm$1.92 &87.52$\pm$2.59 & 90.32$\pm$2.12
    \\ \midrule
    Meta-GNN    & 55.45$\pm$2.15 & 59.71$\pm$2.79 & 61.32$\pm$3.22 & 40.14$\pm$1.94 & 45.52$\pm$1.71 & 35.19$\pm$1.72 & 39.02$\pm$1.99 & 65.15$\pm$2.22 & 70.21$\pm$2.71 & 74.34$\pm$2.41 \\
    GPN         & 57.36$\pm$2.20 & 64.22$\pm$2.92 & 65.59$\pm$2.49 & 42.81$\pm$2.34 & 50.50$\pm$2.13 & 37.36$\pm$1.99 & 42.16$\pm$2.19 & 73.19$\pm$2.10 & 80.70$\pm$2.41 & 83.21$\pm$2.15 \\
    G-Meta      & 54.39$\pm$2.19 & 57.59$\pm$2.42 & 62.49$\pm$2.30 & 40.48$\pm$1.70 & 47.16$\pm$1.73 & 35.49$\pm$2.12 & 40.95$\pm$2.70 & 78.19$\pm$2.65 & \underline{88.68$\pm$1.68} & 92.16$\pm$1.14 \\
    TENT        & 60.03$\pm$3.11 & 65.20$\pm$3.19 & 67.59$\pm$2.95 & 50.26$\pm$1.73 & 61.38$\pm$1.72 & 42.19$\pm$1.16 & 46.29$\pm$1.29 & 51.92$\pm$2.35 & 55.65$\pm$2.19 & 58.30$\pm$2.05 \\
    Meta-GPS    & 58.95$\pm$2.12 & 69.95$\pm$2.02 & 72.56$\pm$2.06 & 52.16$\pm$2.01 & 62.55$\pm$1.95 & 42.96$\pm$2.02 & 46.86$\pm$2.10 & \underline{79.12$\pm$2.66} & 87.91$\pm$2.12 & \underline{91.66$\pm$2.52} \\
    X-FNC       & 58.79$\pm$2.56 & 67.96$\pm$3.10 & 70.29$\pm$3.05 & 52.36$\pm$2.75 & 63.19$\pm$2.22 & 41.92$\pm$2.72 & 46.10$\pm$2.16 & 75.22$\pm$2.59 & 85.62$\pm$3.12 & 90.36$\pm$2.99 \\
    TEG         & 59.70$\pm$2.69 & \underline{73.79$\pm$1.59} & \underline{76.79$\pm$2.12} & \underline{57.35$\pm$1.14} & 62.07$\pm$1.72 & \underline{47.41$\pm$0.63} & \underline{51.11$\pm$0.73} & 61.35$\pm$2.16 & 67.90$\pm$2.10 & 71.10$\pm$2.02 \\
    COSMIC      & 60.95$\pm$2.75 & 70.22$\pm$2.56 & 75.10$\pm$2.30 & 52.98$\pm$2.19 & \underline{65.42$\pm$1.69} & 43.19$\pm$2.72 & 47.59$\pm$2.19 & 62.19$\pm$2.52 & 66.02$\pm$2.29 & 72.10$\pm$2.25 \\
    TLP         & \underline{61.12$\pm$2.10} & 71.10$\pm$2.17 & 75.55$\pm$2.03 & 41.96$\pm$2.29 & 52.99$\pm$2.05 & 39.42$\pm$2.15 & 42.62$\pm$2.09 & 73.19$\pm$2.16 & 85.32$\pm$2.12 & 58.30$\pm$2.05 \\
    Meta-BP     & 60.15$\pm$2.45 & 72.19$\pm$3.19 & 76.11$\pm$3.29 & 55.12$\pm$4.12 & 65.39$\pm$4.55 & 46.25$\pm$4.52 & 50.12$\pm$3.39 & 67.12$\pm$2.39 & 79.15$\pm$2.56 & 85.19$\pm$2.29 \\
    SMILE                                                & 61.05$\pm$3.46              & 72.22$\pm$1.92 & 75.59$\pm$2.49                                   & 55.22$\pm$1.53 & 64.25$\pm$2.12 & 47.36$\pm$1.99 & 50.16$\pm$2.19                                   & 78.85$\pm$2.80              & 85.72$\pm$2.42 & 89.22$\pm$2.11              
    \\
    \midrule
    \textbf{VISION} & \textbf{64.27$\pm$2.99}    & \textbf{74.09$\pm$2.17}    & \multicolumn{1}{c|}{\textbf{77.78$\pm$2.00}}    & \textbf{62.61$\pm$1.97}    & \textbf{66.62$\pm$1.72}    & \textbf{49.61$\pm$1.14}    & \multicolumn{1}{c|}{\textbf{52.79$\pm$1.00}}    & \textbf{85.44$\pm$2.35}    & \textbf{91.60$\pm$1.27}    & \textbf{93.71$\pm$1.03} \\
    \bottomrule
  \end{tabular}%
  }
\end{table*}

We conduct extensive experiments on six benchmark datasets. As shown in Tables \ref{results-merged-1} and \ref{results-merged-2}, our proposed VISION framework achieves consistently superior performance, establishing new state-of-the-art results across all experimental configurations, demonstrating its effectiveness and robustness for few-shot node classification tasks. We attribute this superior performance to two core designs. First, the context-aware network creates deeply task-aware node representations by capturing local topology and then performing global in-context reasoning over the entire task as a sequence. Second, the unsupervised meta-training strategy enhances generalization by constructing high-quality pseudo-tasks from vast unlabeled data, which overcomes the reliance on base labeled data commonly used in many existing graph FSL methods.

Moreover, the results clearly show a general trend where graph in-context learning and graph meta-learning models, as two graph learning paradigms, exhibit superior performance over other types of baselines. This is expected, as the former can effectively leverage rich contextual knowledge, whereas the latter is specifically designed to address the unique challenges of graph FSL. In contrast, traditional graph embedding methods achieve inferior performance due to their lack of mechanisms to handle label scarcity, while general-purpose meta-learning models are limited by their inability to exploit the inherent structural information of graphs.

\begin{table*}[ht]
\centering
\caption{Results of different model variants on all datasets.}
\label{Ablation}
\footnotesize
\begin{tabular}{l|cccccc}
\toprule
\multirow{2}{*}{\textbf{Model}} & \textbf{CoraFull} & \textbf{Cora} & \textbf{Coauthor-CS} & \textbf{CiteSeer} & \textbf{ogbn-arxiv} & \textbf{Cora-ML} \\
\cmidrule{2-7}  
& 5 way 5 shot & 2 way 5 shot & 5 way 5 shot & 2 way 3 shot & 5 way 5 shot & 2 way 1 shot \\
\midrule
\textit{w/o both}         & 72.17$\pm$1.78          & 71.85$\pm$1.81          & 83.34$\pm$1.07          & 65.37$\pm$2.65          & 47.91$\pm$1.80          & 66.05$\pm$2.71          \\
\textit{w/o local}        & 72.65$\pm$1.81          & 66.57$\pm$2.08          & 81.84$\pm$1.07          & 59.02$\pm$2.16          & 47.79$\pm$1.75          & 55.83$\pm$2.16          \\
\textit{w/o global}       & 79.55$\pm$1.63          & 82.22$\pm$1.72          & 89.85$\pm$0.84          & 70.10$\pm$2.60          & 53.42$\pm$1.86          & 65.31$\pm$2.84          \\
\textit{w/o task context} & 79.94$\pm$1.75          & 85.83$\pm$1.58          & 91.24$\pm$0.79          & 64.41$\pm$2.57          & 54.22$\pm$1.81          & 67.93$\pm$3.00          \\
Ours                      & \textbf{83.02$\pm$1.48} & \textbf{86.02$\pm$1.63} & \textbf{94.56$\pm$0.64} & \textbf{74.09$\pm$2.17} & \textbf{66.62$\pm$1.72} & \textbf{85.44$\pm$2.35} \\
\bottomrule
\end{tabular}
\end{table*}

\subsection{Ablation Study}
\label{sec:ablation}
To assess the effectiveness of our proposed context-aware architecture, we design four variants of our VISION model for comparison under various few-shot settings.
(I) \textit{w/o task context}: This variant removes the task-contextualized initialization, relying solely on the feature encoder without explicit support/query role differentiation.
(II) \textit{w/o global}: This variant removes the global in-context reasoning module, preventing the model from explicitly attending to the support set for task-level dependencies.
(III) \textit{w/o local}: This variant discards the local neighborhood aggregation module, feeding the node features directly into the subsequent modules, thereby ignoring local graph topology.
(IV) \textit{w/o both}: This variant omits both the local aggregation and global context modules, effectively degenerating into a standard prototype-based metric learning framework.
We present a summary of ablation results in Table~\ref{Ablation}, with comprehensive results provided in \textbf{Appendix}~\ref{sec:appendix_ablation}.

Based on the results, we have the following in-depth analysis.
(I) Our full VISION model demonstrates superior performance across all settings.
This consistent outperformance against all variants serves as strong evidence for the effectiveness of our holistic design.
(II) The \textit{w/o both} variant generally establishes a performance lower bound across most datasets, which starkly highlights the necessity of our proposed interaction modules.
(III) The performance gaps observed in \textit{w/o task context}, \textit{w/o global}, and \textit{w/o local} further validate our architecture.
The results indicate that neither local structural information nor global task context alone is sufficient.
Specifically, the local aggregation captures topological dependencies, while the task-contextualized initialization and global reasoning ensure our model explicitly understands the task structure.
Removing any of these components leads to suboptimal performance. 

\begin{figure*}[htbp]
    \centering
    \includegraphics[width=0.85\textwidth]{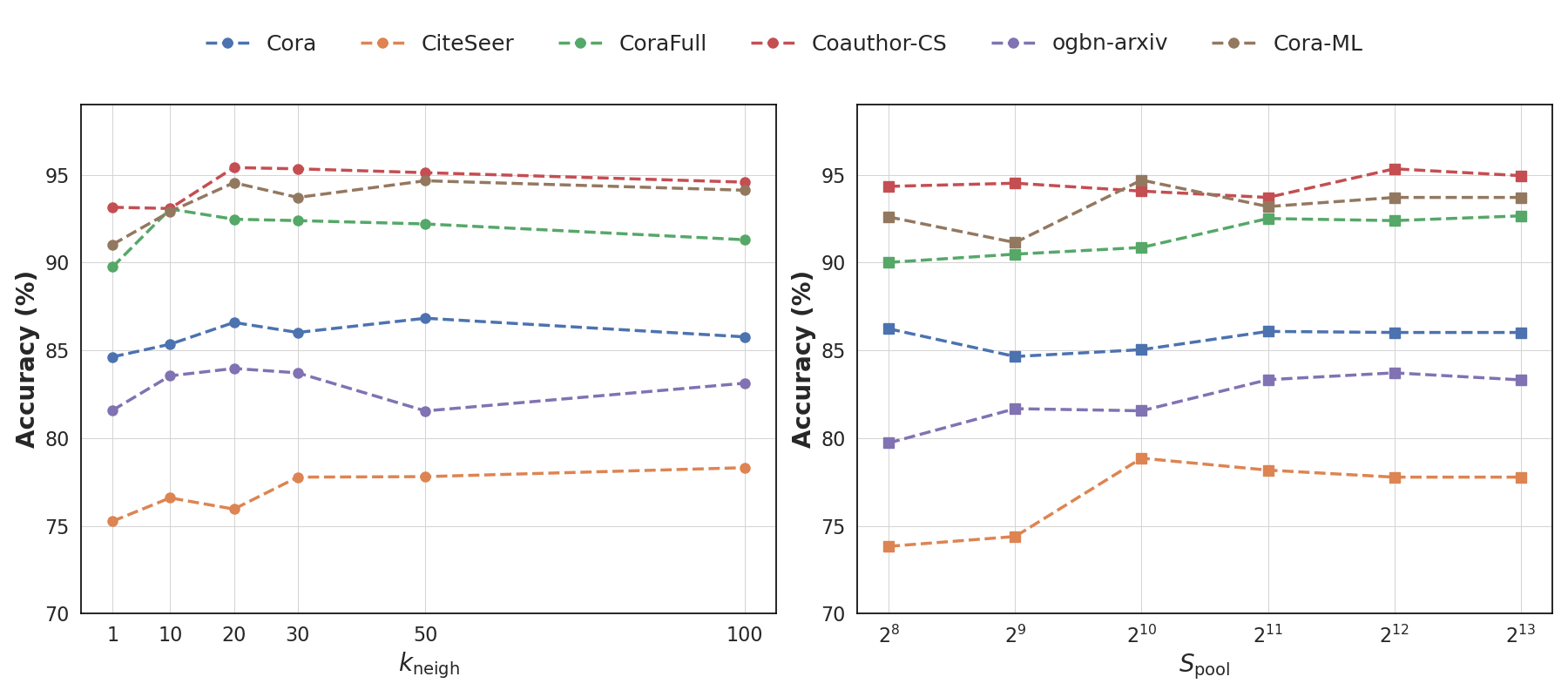}
    \caption{Sensitivity analysis of model performance. The left panel shows the impact of the number of sampled neighbors ($k_{\text{neigh}}$), and the right panel shows the impact of the candidate pool size ($S_{\text{pool}}$).}
    \label{fig:hyperparameter}
\end{figure*}

\subsection{Hyperparameter Sensitivity Analysis}
We analyze the impact of two key hyperparameters---the number of neighbors $k_{\text{neigh}}$ and the candidate pool size $S_{\text{pool}}$---on model performance under the 2-way 5-shot setting, as shown in Fig. \ref{fig:hyperparameter}.
The model is slightly sensitive to $k_{\text{neigh}}$, with performance generally exhibiting a rise-then-saturate trend: too few neighbors fail to capture sufficient local context, while too many can introduce noise and unnecessary computational overhead.
In contrast, the model shows little sensitivity to the candidate pool size $S_{\text{pool}}$, with performance remaining relatively stable across the tested range.
A plausible explanation is that even a moderately sized pool provides sufficient diversity for generating high-quality pseudo-tasks.
Based on this analysis, we select $k_{\text{neigh}}=30$ to balance information capture with efficiency.
For $S_{\text{pool}}$, given that performance does not degrade with a larger pool, we use $S_{\text{pool}}=4,096$ as a robust default setting to ensure sufficient candidate diversity.

%% file: conclusion.tex
\section{Conclusion}
\label{sec:conclusion}

In this work, we propose VISION, a novel framework that pioneers a new paradigm for graph FSL centered on context-aware inference. Specifically, our model incorporates two key techniques. First, we introduce a context-aware architecture that enables the model to better leverage contextual information. It accomplishes classification in a single forward pass without any fine-tuning by framing the task as an in-context learning problem, where it first captures local graph structure and then globally reasons over the entire task as a sequence to dynamically generate query representations. Second, a fully unsupervised meta-training strategy is designed to leverage vast unlabeled data. By generating structure-adaptive features and constructing diverse pseudo-tasks based on similarity, the model learns a generalizable classification ability without reliance on base class labels. Theoretically, our model provides a generalization error bound that supports the advantages of our context-aware design. Empirically, VISION demonstrates state-of-the-art performance across six benchmark datasets, surpassing many competitive baselines. 

%% file: appendix.tex
\section{Appendix}

\subsection{Verification of Random Anchor Selection}
\label{sec:anchor_verification}

A core assumption of our pseudo-task generation is that randomly selected anchors can represent diverse semantics to form discriminative classification tasks.
To mathematically validate this assumption, we adopt the statistical model proposed in previous unsupervised meta-learning work~\cite{siavash2019unsupervised}.
Let $C_{\text{total}}$ be the total number of latent classes in the dataset, and $N$ be the number of anchors sampled (consistent with the $N$-way setting).
Assuming the dataset contains a sufficiently large number of samples per class, the probability $P$ that all $N$ randomly sampled anchors belong to distinct classes is given by the ratio of permutations to total combinations:
\begin{equation}
\label{eq:prob_exact}
P = \frac{P(C_{\text{total}}, N)}{(C_{\text{total}})^N} = \frac{C_{\text{total}} \times (C_{\text{total}}-1) \times \dots \times (C_{\text{total}}-N+1)}{(C_{\text{total}})^N}.
\end{equation}
Although the condition $C_{\text{total}} \gg N$ assumed in prior works does not strictly hold for all our datasets, we demonstrate through exact calculation that the probability of selecting distinct classes remains sufficiently high to support effective meta-learning.
For the Cora dataset ($C_{\text{total}}=7$, $N=2$), the exact calculation is:
\begin{equation}
P_{\text{Cora}} = \frac{7 \times 6}{7^2} = \frac{42}{49} \approx 85.71\%.
\end{equation}
For the ogbn-arxiv dataset ($C_{\text{total}}=40$, $N=5$), the probability is:
\begin{equation}
P_{\text{arxiv}} = \frac{40 \times 39 \times 38 \times 37 \times 36}{40^5} = \frac{78,960,960}{102,400,000} \approx 77.11\%.
\end{equation}
Similarly, for the CoraFull dataset ($C_{\text{total}}=70$, $N=5$), we have:
\begin{equation}
P_{\text{Full}} = \frac{70 \times 69 \times 68 \times 67 \times 66}{70^5} = \frac{1,452,361,680}{1,680,700,000} \approx 86.41\%.
\end{equation}
These mathematical results confirm that simple random sampling consistently yields a high degree of class diversity. Even in cases where collisions occur, they serve as a form of label noise that can enhance model robustness.

\begin{table}[ht]
\centering
\caption{Statistical verification of the diversity of randomly selected anchors.}
\label{tab:anchor_stats}
\resizebox{\columnwidth}{!}{%
\begin{tabular}{lccc}
\toprule
\textbf{Dataset} & \textbf{Total Classes} & \textbf{Task Setting ($N$-way)} & \textbf{Avg. Distinct Classes (Max=$N$)} \\
\midrule
Cora        & 7   & 2 & 1.63 \\
CiteSeer    & 6   & 2 & 1.50 \\
Cora-ML     & 7   & 2 & 1.66 \\
CoraFull    & 70  & 5 & 4.64 \\
ogbn-arxiv  & 40  & 5 & 4.19 \\
Coauthor-CS & 15  & 5 & 3.25 \\
\bottomrule
\end{tabular}%
}
\end{table}

To further verify this empirically, we conduct a statistical analysis on the class composition of randomly selected anchors across all datasets.
We simulate the anchor selection process for 10,000 episodes under the standard $N$-way setting used in our experiments.
Table \ref{tab:anchor_stats} reports the Average Distinct Classes metric, which quantifies the average number of unique true classes covered by the $N$ randomly sampled anchors.
As shown in Table \ref{tab:anchor_stats}, the results demonstrate that random sampling consistently yields a high degree of class diversity, ensuring that the constructed pseudo-tasks are sufficiently discriminative for training.

\subsection{Theoretical Proofs}
\label{sec:appendix_proofs}

In this section, we provide the detailed proofs for the theoretical results presented in Section \ref{sec:theory}. We adopt the standard learning theory framework for domain adaptation and meta-learning.

\subsubsection{Proof of Theorem \ref{thm:main}}

\begin{proof}
Let $\mathcal{L}(f, \mathcal{T})$ denote the loss of model $f$ on a task $\mathcal{T}$. The generalization error on the real task distribution is $\epsilon_{\text{gen}}(f) = \mathbb{E}_{\mathcal{T} \sim P_{\text{real}}}[\mathcal{L}(f, \mathcal{T})]$.
The empirical risk on pseudo-tasks is $\hat{\epsilon}_\text{pseudo}(f) = \frac{1}{T} \sum_{i=1}^T \mathcal{L}(f, \mathcal{T}_i^\text{pseudo})$.

We decompose the error using the triangular inequality, introducing the expected risk on the pseudo-distribution $\epsilon_{\text{pseudo}}(f)$:
\begin{equation}
\epsilon_{\text{gen}}(\hat{f}) \leq \epsilon_{\text{pseudo}}(\hat{f}) + |\epsilon_{\text{gen}}(\hat{f}) - \epsilon_{\text{pseudo}}(\hat{f})|.
\end{equation}

\noindent \textbf{Step 1: Bounding the Meta-Learning Error (Term A).} According to standard statistical learning theory \cite{bartlett2002rademacher}, the gap between empirical and expected risk on the source distribution (pseudo-tasks) is bounded by the Rademacher complexity of the hypothesis class $\mathcal{H}$:
\begin{equation}
\epsilon_{\text{pseudo}}(\hat{f}) \leq \hat{\epsilon}_\text{pseudo}(\hat{f}) + 2\mathcal{R}_T(\mathcal{H}) + \sqrt{\frac{\ln(1/\delta)}{2T}}.
\end{equation}
This corresponds to Term (A) in Eq. \ref{eq:main_bound}, representing the estimation error due to finite training tasks.

\noindent \textbf{Step 2: Bounding the Distribution Shift (Term B).} The term $|\epsilon_{\text{gen}}(\hat{f}) - \epsilon_{\text{pseudo}}(\hat{f})|$ represents the discrepancy between training on pseudo-tasks and testing on real tasks.
Since the loss function $\mathcal{L}$ is assumed to be $L$-Lipschitz continuous with respect to the task representations (which is satisfied by our use of bounded activation functions like Tanh/GELU and LayerNorm), we can invoke the Kantorovich-Rubinstein duality of the Wasserstein-1 distance $W_1$:
\begin{equation}
|\mathbb{E}_{P_{\text{real}}}[\mathcal{L}(f, \cdot)] - \mathbb{E}_{P_{\text{pseudo}}}[\mathcal{L}(f, \cdot)]| \leq L \cdot W_1(P_{\text{real}}^{\phi}, P_{\text{pseudo}}^{\phi}).
\end{equation}
Here, $P^{\phi}$ denotes the distribution of task features transformed by our feature extractor and task-level normalization (TLN).
In our method, TLN centers the features of each task: $\mathbf{x}' = \mathbf{x} - \mu_{\mathcal{T}}$. This operation aligns the first-order statistics of arbitrary graph tasks to a common origin, considerably reducing the Wasserstein distance $W_1$ compared to raw features. This theoretically justifies the necessity of the TLN module in our code.

Combining Step 1 and Step 2 yields the final bound in Theorem \ref{thm:main}.
\end{proof}

\subsubsection{Proof of Corollary \ref{coro:icl}}

\begin{proof}
We analyze the variance of the node representation to demonstrate the benefit of the Global Task Attention module.

Let $\mathbf{h}_q$ be the representation of a query node.
In a model without global context (only local aggregation), $\mathbf{h}_q^\text{local}$ depends only on the node's neighborhood $\mathcal{N}(q)$. The variance of this estimator, conditioned on the true class mean $\mu_c$, is denoted as $\text{Var}_\text{local} = \mathbb{E}[||\mathbf{h}_q^\text{local} - \mu_c||^2]$.

In VISION, the dual-context fusion module updates the local representation by attending to the support set $\mathcal{S}_c = \{\mathbf{s}_1, ..., \mathbf{s}_K\}$ of class $c$. Assuming uniform attention for simplicity, the global context vector is computed as follows:
\begin{equation}
\mathbf{h}_q^{global} = \sum_{k=1}^K \alpha_k \mathbf{s}_k \approx \frac{1}{K} \sum_{k=1}^K \mathbf{s}_k.
\end{equation}
Assuming the support samples are independent observations of the class prototype with noise variance $\sigma^2$ (\textit{i.e.}, $\mathbf{s}_k \sim \mathcal{N}(\mu_c, \sigma^2\mathbf{I})$), the variance of the averaged global context is:
\begin{equation}
\text{Var}(\mathbf{h}_q^{global}) \approx \frac{\sigma^2}{K}.
\end{equation}

The final representation is a fusion: $\mathbf{z}_q = (1-g)\mathbf{h}_q^\text{local} + g\mathbf{h}_q^\text{global}$.
Assuming partial independence between local structural noise and global semantic noise, the variance of the fused estimator is dominated by the term with the lower variance.
Since $\frac{\sigma^2}{K}$ decreases linearly with the shot number $K$, the global attention mechanism effectively acts as a variance reduction estimator.
Specifically, compared to $\text{Var}_\text{local}$, the fused variance is reduced by a term proportional to the information gain from the $K$ support shots. This proves that the in-context learning mechanism strictly reduces predictive uncertainty compared to local-only baselines.
\end{proof}

\begin{table*}[htbp]
\centering
\caption{Ablation study results (Part 1): CoraFull, Coauthor-CS, and Cora. The best result is highlighted in bold.}
\label{tab:ablation_part1_equal}
\resizebox{0.9\textwidth}{!}{%
\begin{tabular}{c|ccc|ccc|cc}
\toprule
\multirow{2}{*}{\textbf{Model}} & \multicolumn{3}{c|}{\textbf{CoraFull}} & \multicolumn{3}{c|}{\textbf{Coauthor-CS}} & \multicolumn{2}{c}{\textbf{Cora}} \\ \cmidrule(lr){2-9}
& 5 way 3 shot & 10 way 3 shot & 10 way 5 shot & 2 way 3 shot & 2 way 5 shot & 5 way 3 shot & 2 way 1 shot & 2 way 3 shot \\ \midrule
\textit{w/o both}         & 69.04$\pm$2.04            & 60.64$\pm$1.36            & 62.92$\pm$1.29            & 83.66$\pm$2.35            & 84.96$\pm$2.10            & 78.73$\pm$1.33            & 59.13$\pm$2.75            & 76.58$\pm$2.07            \\
\textit{w/o local}        & 68.24$\pm$1.98            & 60.44$\pm$1.34            & 63.70$\pm$1.27            & 77.56$\pm$2.41            & 83.85$\pm$2.45            & 78.01$\pm$1.25            & 55.21$\pm$2.13            & 59.75$\pm$2.19            \\
\textit{w/o global}       & 75.18$\pm$1.95            & 66.37$\pm$1.29            & 70.27$\pm$1.14            & 87.61$\pm$2.15            & 89.84$\pm$2.03            & 87.36$\pm$0.99            & 60.30$\pm$2.56            & 75.96$\pm$2.04            \\
\textit{w/o task context} & 74.59$\pm$1.92            & 66.85$\pm$1.27            & 69.76$\pm$1.27            & 88.91$\pm$2.17            & 89.63$\pm$2.01            & 89.46$\pm$0.93            & 59.77$\pm$2.32            & 75.02$\pm$2.04            \\
Ours                      & \textbf{80.23$\pm$1.80}   & \textbf{70.73$\pm$1.32}   & \textbf{73.64$\pm$1.10}   & \textbf{94.14$\pm$1.39}   & \textbf{95.34$\pm$1.22}   & \textbf{93.20$\pm$0.74}   & \textbf{76.11$\pm$2.80}   & \textbf{82.72$\pm$1.88}   \\
\bottomrule
\end{tabular}
}
\end{table*}

\begin{table*}[htbp]
\centering
\caption{Ablation study results (Part 2): CiteSeer, ogbn-arxiv, and Cora-ML. The best result is highlighted in bold.}
\label{tab:ablation_part2_equal}
\resizebox{0.9\textwidth}{!}{%
\begin{tabular}{c|cc|ccc|cc}
\toprule
\multirow{2}{*}{\textbf{Model}} & \multicolumn{2}{c|}{\textbf{CiteSeer}} & \multicolumn{3}{c|}{\textbf{ogbn-arxiv}} & \multicolumn{2}{c}{\textbf{Cora-ML}} \\ \cmidrule(lr){2-8}
& 2 way 1 shot & 2 way 5 shot & 5 way 3 shot & 10 way 3 shot & 10 way 5 shot & 2 way 3 shot & 2 way 5 shot \\ \midrule
\textit{w/o both}         & 56.72$\pm$2.66            & 66.55$\pm$2.29            & 44.82$\pm$1.70            & 30.59$\pm$0.98            & 32.98$\pm$0.97            & 79.72$\pm$1.94            & 87.89$\pm$1.52            \\
\textit{w/o local}        & 52.36$\pm$2.07            & 66.65$\pm$2.33            & 43.68$\pm$1.67            & 29.77$\pm$1.00            & 33.96$\pm$0.98            & 68.62$\pm$2.12            & 84.70$\pm$1.60            \\
\textit{w/o global}       & 57.72$\pm$2.45            & 65.73$\pm$2.25            & 49.82$\pm$1.99            & 37.82$\pm$1.08            & 40.23$\pm$0.98            & 90.98$\pm$1.35            & 81.08$\pm$1.73            \\
\textit{w/o task context} & 58.89$\pm$2.87            & 66.88$\pm$2.25            & 50.32$\pm$2.04            & 35.72$\pm$1.04            & 42.03$\pm$1.02            & 89.95$\pm$1.32            & 90.31$\pm$1.30            \\
Ours                      & \textbf{64.27$\pm$2.99}   & \textbf{77.78$\pm$2.00}   & \textbf{62.61$\pm$1.97}   & \textbf{49.61$\pm$1.14}   & \textbf{52.79$\pm$1.00}   & \textbf{91.60$\pm$1.27}   & \textbf{93.71$\pm$1.03}   \\
\bottomrule
\end{tabular}
}
\end{table*}

\subsection{Dataset Descriptions}
\label{Dataset_Descriptions}
We evaluate our approach on a diverse collection of graph datasets spanning multiple domains. For each dataset, the label space is partitioned into disjoint class sets designated for meta-training, meta-validation, and meta-testing. The details are as follows:\\
\textbf{CoraFull} \cite{bojchevski2017deep}: This dataset extends the classic Cora citation network by incorporating a wider set of categories. Nodes denote academic papers, while edges capture citation relationships. We allocate 40 classes for meta-training, 15 classes for validation, and 15 classes for testing.\\
\textbf{Cora} \cite{yang2016revisiting}: This citation network represents academic papers as nodes and their citation relationships as edges. Each node is annotated with a label corresponding to the research topic of the paper. We partition the dataset into 3 classes for meta-training, 2 classes for meta-validation, and 2 classes for meta-testing.\\
\textbf{Coauthor-CS} \cite{shchur2018pitfalls}: This collaboration network represents authors as nodes, with edges denoting co-authorship relations in the computer science domain. Node labels correspond to the authors’ research specialties. For each meta-learning phase, the dataset is partitioned into 5 classes.\\
\textbf{CiteSeer} \cite{yang2016revisiting}: This dataset is a document-level citation graph, where nodes correspond to scientific publications and edges denote citation links. Each node is labeled according to the thematic domain of the publication. The dataset is divided into 2 classes for each of the meta-training, meta-validation, and meta-testing phases.\\
\textbf{ogbn-arxiv} \cite{hu2020open}: A large-scale citation graph constructed from arXiv papers in the computer science field. Each node represents a publication, and edges reflect citation dependencies. Node labels are determined by subject areas defined in the arXiv taxonomy. The dataset is divided into 20 classes for training, with 10 classes each for validation and testing.\\
\textbf{Cora-ML} \cite{bojchevski2017deep}: This dataset consists exclusively of machine learning papers. Nodes correspond to individual publications, and edges are established through citation relationships. We partition the label space into 3 classes for meta-training, 2 classes for validation, and 2 classes for testing.

\subsection{Baseline Descriptions}
\label{Baseline_Descriptions}

We evaluate VISION against a comprehensive set of baselines, categorized into four groups consistent with our experimental results:

\vspace{0.5em}
\noindent \textbf{Graph Embedding Methods}
\vspace{0.5em}

\noindent $\bullet$ \textbf{DeepWalk} \cite{perozzi2014deepwalk}: This method performs truncated random walks on graphs and, inspired by the word2vec framework, maps nodes into a low-dimensional embedding space.

\noindent $\bullet$ \textbf{node2vec} \cite{grover2016node2vec}: An extension of DeepWalk that employs a flexible biased random walk strategy (combining BFS and DFS) to explore diverse neighborhoods and capture both structural equivalence and homophily.

\noindent $\bullet$ \textbf{GCN} \cite{kipf2016semi}: It applies a first-order approximation of spectral graph convolutions to generate hidden node embeddings, which can then be used for downstream learning tasks.

\noindent $\bullet$ \textbf{SGC} \cite{wu2019simplifying}: By removing nonlinearities and collapsing weight matrices, it simplifies the standard GCN architecture, resulting in an efficient yet competitive linear model.



\vspace{0.5em}
\noindent \textbf{Meta-Learning Methods}
\vspace{0.5em}

\noindent $\bullet$ \textbf{ProtoNet} \cite{snell2017prototypical}: This approach constructs a metric space where query samples are classified by comparing them to class prototypes computed as the mean of support set embeddings.

\noindent $\bullet$ \textbf{MAML} \cite{finn2017model}: It learns a set of initial parameters such that only a few gradient updates are needed for the model to adapt to new tasks with scarce labeled data, enabling rapid task generalization.


\vspace{0.5em}
\noindent \textbf{Graph In-Context Learning Methods}
\vspace{0.5em}

\noindent $\bullet$ \textbf{VNT} \cite{tan2023virtual}: Virtual Node Tuning (VNT) injects virtual nodes as soft prompts into the embedding space of a pre-trained Graph Transformer. These prompts are tuned with few-shot support samples to modulate node embeddings for specific tasks without fine-tuning the entire encoder.

\noindent $\bullet$ \textbf{GraphPrompt} \cite{liu2023graphprompt}: This framework unifies pre-training and downstream tasks into a common template (\textit{e.g.}, subgraph similarity) and employs a learnable prompt vector to assist the downstream task in retrieving relevant knowledge from the pre-trained model.

\vspace{0.5em}
\noindent \textbf{Graph Meta-Learning Methods}
\vspace{0.5em}

\noindent $\bullet$ \textbf{Meta-GNN} \cite{zhou2019meta}: It integrates MAML with GNNs, leveraging meta-gradients to update the GNN parameters for rapid adaptation to new few-shot node classification tasks.

\noindent $\bullet$ \textbf{GPN} \cite{ding2020graph}: It extends ProtoNet by incorporating a graph encoder and an evaluator, which jointly learn node embeddings, estimate node importance, and perform prototype-based classification.

\noindent $\bullet$ \textbf{G-Meta} \cite{huang2020graph}: It constructs node-centered subgraphs to capture local structural information and leverages meta-gradients to transfer knowledge across different tasks.

\noindent $\bullet$ \textbf{TENT} \cite{wang2022task}: It proposes a task-adaptive framework consisting of node-level, class-level, and task-level components to reduce the gap between meta-training and meta-testing. 

\noindent $\bullet$ \textbf{Meta-GPS} \cite{liu2022few}: It builds upon MAML by introducing prototype-driven parameter initialization with scaling and shifting, thereby improving knowledge transfer and facilitating faster adaptation.

\noindent $\bullet$ \textbf{X-FNC} \cite{wang2023few}: A framework designed to address label scarcity by extracting meta-knowledge across tasks and leveraging pseudo-labeled nodes from the unlabeled set to enhance the support set.

\noindent $\bullet$ \textbf{TEG} \cite{kim2023task}: It develops a task-equivariant framework that employs equivariant neural networks to extract task-adaptive strategies and encode inductive biases from multiple graph tasks.

\noindent $\bullet$ \textbf{COSMIC} \cite{wang2023contrastive}: It introduces a contrastive meta-learning scheme that aligns node embeddings within each episode via a two-step optimization strategy, enhancing few-shot performance.

\noindent $\bullet$ \textbf{TLP} \cite{tan2022transductive}: It freezes the pre-trained encoder and trains a linear classifier (\textit{e.g.}, logistic regression) on the support set while utilizing the query set statistics for transductive normalization.

\noindent $\bullet$ \textbf{Meta-BP} \cite{zhang2025unlocking}: It presents a lightweight meta-learner that distills relevant knowledge from a pre-trained black-box GNN and exploits task-specific cues for rapid adaptation.

\noindent $\bullet$ \textbf{SMILE} \cite{liu2025dual}: It proposes a dual-level mixup strategy (within-task and across-task) to enrich available nodes and tasks, and explicitly leverages node degree priors to encode expressive representations.

\subsection{More Ablation Study}
\label{sec:appendix_ablation}
This section presents the comprehensive results of the ablation study across all evaluated datasets and few-shot settings.
As described in the main text, we compare our full model, VISION, against four variants (\textit{i.e.}, \textit{w/o task context}, \textit{w/o global}, \textit{w/o local}, and \textit{w/o both}) to strictly validate our architectural design.
The detailed results are presented in Tables~\ref{tab:ablation_part1_equal} and~\ref{tab:ablation_part2_equal}.